\documentclass[mnsc,nonblindrev]{informs3_hide}
\OneAndAHalfSpacedXI 

\usepackage[english]{babel}
\usepackage[autostyle, english = american]{csquotes}
\MakeOuterQuote{"}
\usepackage[colorlinks,linkcolor=blue, citecolor=blue]{hyperref}
\usepackage[normalem]{ulem}

\usepackage{cleveref}
\crefname{subsection}{section}{subsections}
\usepackage{eqnarray}
\usepackage{algorithm,algpseudocode}
\usepackage{amsmath, bbm, enumitem, xparse, bm, lipsum}
\usepackage{amssymb}
\usepackage{comment}
\usepackage{mathtools}
\usepackage{subcaption}
\usepackage{booktabs}

\newcommand{\ALG}{\mathsf{ALG}}
\newcommand{\OPT}{\mathsf{OPT}}

\newcommand{\Reg}{\operatorname{Reg}}

\Crefname{assumption}{Assumption}{Assumptions}
\crefname{assumption}{assumption}{assumptions}
\newcommand{\Nmax}{N_{\max}}

\NewDocumentEnvironment{myproof}{o}
{\IfNoValueTF{#1}{\paragraph{{Proof.} }} {\paragraph{{#1.} }} }
{\hfill$\qed$}

\usepackage[dvipsnames]{xcolor}

\usepackage{xparse}

\usepackage{natbib,bbm,multirow,multicol}
 \bibpunct[, ]{(}{)}{,}{a}{}{,}%
 \def\bibfont{\small}%
\TheoremsNumberedThrough     
\ECRepeatTheorems

\EquationsNumberedThrough    

\allowdisplaybreaks

\begin{document}


\RUNAUTHOR{Zong, Wang, Jiang}

\RUNTITLE{Adaptive Rollout Optimization}

\TITLE{
Cross-Epoch Adaptive Rollout Optimization for RL Post-Training
}

\ARTICLEAUTHORS{%
\AUTHOR{Yiming Zong,\quad Yige Wang, \quad Jiashuo Jiang}

\AFF{\  \\
Department of Industrial Engineering \& Decision Analytics, Hong Kong University of Science and Technology
}
}

\ABSTRACT{
LLM post-training often relies on reinforcement learning methods that sample multiple rollouts per prompt, yet most existing approaches use a fixed rollout budget for every prompt, despite large differences in the training signal different prompts provide. In this paper, we study adaptive rollout allocation under a fixed global budget and formulate the problem as online resource allocation with prompt-level diminishing returns. Our method, \textsc{CERO}, maintains a Beta posterior over each prompt’s success probability and uses the posterior expected Bernoulli variance as a Bayesian estimate of the value of additional rollouts. We use this estimate to construct a concave, saturating utility over cumulative allocations, yielding an objective in which decisions across prompts and epochs are coupled by the global budget. Since the resulting objective is temporally nonseparable, we derive a Fenchel-dual reformulation and update both prompt-level and budget-level dual variables via projected online gradient descent. Under fixed prompt utilities, we prove an $O(\sqrt K)$ regret bound against the offline allocation benchmark. Experiments on mathematical-reasoning problems show that CERO consistently outperforms GRPO across multiple open-weight LLMs and benchmarks, demonstrating that adaptive rollout budgeting can improve sample efficiency.
}



\maketitle

\section{Introduction}
Reinforcement learning (RL) post-training has become a central ingredient for improving the reasoning ability of large language models (LLMs). In widely used algorithms such as GRPO \citep{shao2024deepseekmath}, each prompt is typically assigned a fixed number of rollouts, which are grouped to construct relative reward signals for policy optimization. While simple and scalable, this uniform allocation strategy ignores substantial heterogeneity across prompts. Some prompts are already solved and yield little reward variation; others are too difficult and provide weak or noisy signals; only a subset of prompts produce useful success--failure variation that can drive meaningful policy updates. As a result, a large fraction of the rollout budget may be spent on prompts with limited marginal value.

Recent rollout-optimization methods address this inefficiency through prompt selection \citep{yu2025dapo}, within-batch reallocation \citep{li2025knapsack}, or post-generation filtering and pruning. These methods improve training efficiency, but they primarily make local decisions within a batch or after rollouts have already been generated. By contrast, rollout budgeting in RL post-training naturally spans multiple epochs: decisions made early in training affect the remaining budget, the evidence collected for each prompt, and the opportunities for future updates. This motivates a global question: under a fixed rollout budget, how should we allocate rollouts across both prompts and epochs to maximize the utility of collected samples?

We formulate this problem as budgeted online optimization with prompt-level diminishing returns. An effective allocation rule should favor prompts with high expected learning value, while avoiding excessive sampling of the same prompt as evidence accumulates. To this end, we propose \textsc{CERO}, a cross-epoch online adaptive rollout optimization framework. To the best of our knowledge, \textsc{CERO} is the first rollout-allocation framework for LLM post-training that explicitly optimizes a global rollout budget across epochs rather than only reallocating or filtering rollouts within a batch.

\textsc{CERO} maintains a Beta posterior over each prompt's latent success probability and uses the posterior expected Bernoulli reward variance to estimate the value of additional rollouts. This estimate induces a concave, saturating utility over cumulative prompt-level allocations, capturing diminishing returns as more rollouts are assigned to the same prompt. Since this utility depends on cumulative allocations, the resulting objective is temporally non-separable under the global budget. We then derive a Fenchel-dual reformulation of the budgeted allocation problem and obtain an online primal--dual algorithm with prompt-level dual variables and a global budget multiplier. The prompt-level variables estimate the marginal value of additional rollouts, while the global multiplier acts as a dynamic budget price coordinating rollout consumption across epochs.

The resulting algorithm is simple to integrate with existing RL post-training pipelines. At each epoch, \textsc{CERO} allocates rollouts to a prompt when its estimated marginal value exceeds the current global budget price, updates the prompt posterior using observed rewards, and refreshes the plug-in value estimate for future epochs. It acts as a drop-in adaptive data-collection layer on top of GRPO-style policy optimization: it changes which prompts receive rollouts and how many rollouts they receive, but leaves the underlying policy-gradient objective unchanged.

Our contributions are threefold. First, we formulate adaptive rollout allocation for LLM RL post-training as a cross-epoch fixed-budget online optimization problem with prompt-specific diminishing-return utilities. This formulation explicitly couples rollout decisions across the training horizon through a shared global budget, distinguishing it from prior within-batch selection or post-generation filtering methods. Second, we propose \textsc{CERO}, a Bayesian primal--dual allocation algorithm that uses posterior estimates of prompt-level rollout value to guide budgeting across epochs. Third, we provide both theoretical and empirical evidence for the effectiveness of \textsc{CERO}: under fixed utilities induced by prompt-level value estimates, we establish an $O(\sqrt{K})$ regret guarantee against the offline allocation benchmark. Experiments on mathematical-reasoning problems demonstrate that \textsc{CERO} outperforms vanilla GRPO across multiple open-weight LLMs and benchmarks.

\subsection{Related Works}
\paragraph{RL Rollout Optimization} A surge of research has recently emerged focusing on RL rollout optimization. These works predominantly involve modifications to GRPO \citep{shao2024deepseekmath} and can be broadly categorized into three perspectives. The first line focuses on online curriculum learning and sample selection, where training examples are prioritized according to their estimated learning value \citep{mahrooghi2026goldilocks,zhang2025speed,chen2025self,hu2025vade}.
The second line investigates adaptive rollout budget allocation \citep{li2025knapsack,yao2026coba}, which replace uniform rollout assignment with value-aware resource allocation under a fixed total budget. The third line emphasizes post-rollout filtering or pruning, such as AERO \citep{zhang2026train}, CPPO \citep{lin2025cppo}, PODS \citep{xu2025not}, and GFPO \citep{shrivastava2025sample}, which reduce update cost by rescuing, smoothing, or sub-selecting rollout groups after generation. Overall, prior work suggests that rollout efficiency depends critically on identifying high-value queries and avoiding degenerate all-correct or all-incorrect groups.

\paragraph{Online Resource Allocation} Online resource allocation is a central topic in the online optimization literature and has been applied in many operational settings, including ad assignment \citep{mehta2007adwords}, revenue management \citep{talluri2006theory,jasin2012re, jiang2025degeneracy}, online knapsack problems \citep{arlotto2019uniformly, jiang2020onlineRA, liu2022non}, and online packing/covering problems \citep{buchbinder2005online,buchbinder2006improved,feldman2010online}. The existing algorithmic literature can be broadly organized according to how it exploits dual information. A dominant class of methods uses dual prices as signals for making online allocation decisions. Within this class, some papers adopt a learn-once strategy: they reserve an initial set of arrivals for estimation, compute a price vector from that sample, and then use the same prices for all future decisions \citep{devanur2009adwords,feldman2010online,molinaro2014geometry,devanur2012online, zong2026online}. Other papers advocate adaptive updating, where the pricing rule is revised at multiple points over time through repeated re-solving, leading to stronger robustness and improved horizon dependence \citep{agrawal2014dynamic,li2022online,chen2015dynamic}. In parallel, other works depart from this paradigm by developing primal-oriented algorithms \citep{kesselheim2014primal} or first-order schemes that do not require repeated computation of dual prices \citep{agrawal2014fast,balseiro2020dual}.

\section{Problem Formulation}
We focus on the LLM post-training stage and study the optimization of rollout allocation. Consider a training procedure with \(K\) epochs and a training dataset of \(M\) prompts, \(\mathcal X=\{x_1,\ldots,x_M\}\). At each epoch \(k\), each prompt \(x_i\) becomes available once, and the current policy generates multiple rollouts conditioned on \(x_i\) for policy optimization.
A standard approach, such as GRPO, allocates a fixed number \(N\) of rollouts to each prompt in each epoch, yielding a total rollout budget \(B=KMN\). 

In this work, we consider an adaptive rollout allocation strategy. Let
\[
\mathbf N_{\mathcal X}=(\bm N_1,\ldots,\bm N_M)\in\mathbb Z_{\ge0}^{M\times K},
\quad
\bm N_i=(N_{i,1},\ldots,N_{i,K})\in\mathbb Z_{\ge0}^{K},
\]
denote the allocation over the entire training procedure, where \(N_{i,k}\) is the number of rollouts assigned to prompt \(x_i\) at epoch \(k\). The allocation is subject to the global rollout budget
\[
\sum_{k=1}^K\sum_{i=1}^M N_{i,k} \le B.
\]

For the theoretical formulation, we associate each prompt \(x_i\) with a fixed score \(q_i\) that abstracts the expected learning signal obtained from additional rollouts. We refer to \(q_i\) as the prompt informativeness score. The score \(q_i\) parameterizes the prompt-level utility function and determines how valuable it is to allocate additional rollouts to \(x_i\).
In our implementation, \(q_i\) is updated across epochs using a Beta-posterior plug-in rule. Let
\[
U_i(n):=U(q_i,n)
\]
denote the utility of assigning \(n\) cumulative rollouts to prompt \(x_i\). Since \(q_i\) is fixed in the theoretical formulation, we omit it from the notation and view \(U_i\) as a prompt-specific utility function. We will specify informativeness score \(q_i\) and utility $U_i(n)$ in \Cref{sec:informativeness}. The rollout allocation problem is then formulated as
\begin{equation}
\label{eq:allocation}
\begin{aligned}
\OPT(\mathcal X):=\max_{\{N_{i,k}\}} 
\quad & \sum_{i=1}^M U_i\!\left(\sum_{k=1}^K N_{i,k}\right) \\
\text{s.t.}\quad
& \sum_{i=1}^M\sum_{k=1}^K N_{i,k}\le B,\\
& N_{i,k}\in\{0,1,\ldots,\Nmax\},\quad i\in[M],\ k\in[K].
\end{aligned}
\end{equation}
Here \(\sum_{k=1}^K N_{i,k}\) is the cumulative number of rollouts assigned to prompt \(x_i\) and \(\Nmax\) denotes the maximum number of rollouts that can be assigned to a single prompt in one epoch. Compared with uniform allocation, this formulation allows rollout numbers to vary across prompts and epochs, so that more budget can be assigned to prompts with higher utility while respecting the same total budget.

\subsection{Problem Reformulation}

The allocation problem in \Cref{eq:allocation} is coupled across epochs because the utility of prompt \(x_i\) depends on its cumulative allocation \(\sum_{k=1}^K N_{i,k}\). This temporal coupling makes it difficult to choose the epoch-wise allocations \(N_{i,k}\) in an online manner. We address this difficulty by using a Fenchel-dual representation of each utility function, which linearizes the dependence on cumulative allocations and yields an allocation subproblem that decomposes across epochs.

For each prompt \(x_i\), define the concave conjugate of \(U_i\) as
\[
U_i^*(\theta)
=
\inf_{s\ge 0}\bigl\{s\theta - U_i(s)\bigr\}.
\]
Assuming that \(U_i\) is closed, concave and nondecreasing on \(\mathbb{R}_+\), we have the biconjugate representation
\[
U_i(s)
=
\inf_{\theta\ge 0}\bigl\{s\theta - U_i^*(\theta)\bigr\}.
\]
Introducing a Fenchel dual variable \(\theta_i\) for each prompt \(x_i\), together with a Lagrange multiplier \(\mu\ge 0\) for the global rollout budget constraint, we obtain the following dual objective:
\[
L_{\mathcal X}(\bm\theta,\mu)
=
B\mu
+
\max_{\{N_{i,k}\in\{0,1,\ldots,\Nmax\}\}}
\left\{
\sum_{i=1}^M
(\theta_i-\mu)
\sum_{k=1}^K N_{i,k}
\right\}
-
\sum_{i=1}^M U_i^*(\theta_i).
\]
By weak duality, we know that
\[
\OPT(\mathcal X)
\le
\inf_{\bm\theta\ge 0,\mu\ge 0}
L_{\mathcal X}(\bm\theta,\mu).
\]

The key advantage of this reformulation is that the maximization over allocations now decomposes across epochs. For fixed \((\bm\theta,\mu)\), the allocation term can be written as \(K\) identical per-epoch subproblems. Hence,
\[
L_{\mathcal X}(\bm\theta,\mu)
=
K\,L(\bm\theta,\mu),
\]
where the per-epoch dual objective is
\[
L(\bm\theta,\mu)
=
\frac{B}{K}\mu
+
\max_{\{n_i\in\{0,1,\ldots,\Nmax\}\}}
\left\{
\sum_{i=1}^M (\theta_i-\mu)n_i
\right\}
-
\frac{1}{K}\sum_{i=1}^M U_i^*(\theta_i).
\]
Here, \(n_i\) can be interpreted as the allocation decision for prompt \(x_i\) in a representative epoch. Equivalently, if
\[
\tilde N_i := \frac{1}{K}\sum_{k=1}^K N_{i,k},
\]
then \(\tilde N_i\) denotes the average number of rollouts assigned to prompt \(x_i\) across epochs. Thus, the Fenchel-dual reformulation turns the temporally coupled cumulative-utility objective into per-epoch allocation subproblems coordinated by the dual variables \((\bm\theta,\mu)\).

\section{Informativeness and Utility Function Design} \label{sec:informativeness}
In this section, we specify the prompt informativeness score and the resulting utility function. We write \(q_i\) for the fixed score used in the theoretical formulation and \(q_{k,i}\) for the plug-in posterior score used by the implementation at epoch \(k\).

Consider each prompt \(x_i\), and let \(\mathcal H_i^k\) denote the rollout history collected for this prompt up to the beginning of epoch \(k\). We model its latent success probability \(p_i\) using a Beta posterior:
\[
p_i \mid \mathcal H_i^k \sim \mathrm{Beta}(a_{k,i},b_{k,i}).
\]
This posterior serves as a compact Bayesian summary of the current evidence for prompt \(x_i\), and provides a principled basis for measuring both the current success probability and the remaining uncertainty.
The posterior mean is
\[
\tilde p_{k,i} := \mathbb E[p_i\mid \mathcal H_i^k] =\frac{a_{k,i}}{a_{k,i}+b_{k,i}},
\]
which estimates the current success probability of prompt \(x_i\). The posterior variance is
\[
\sigma_{k,i}^2 := \mathrm{Var}(p_i\mid \mathcal H_i^k) = \frac{a_{k,i}b_{k,i}} {(a_{k,i}+b_{k,i})^2(a_{k,i}+b_{k,i}+1)}.
\]
The posterior mean captures the current performance estimate, while the posterior variance characterizes the remaining uncertainty in that estimate. Prompts with larger posterior uncertainty may be more informative, but uncertainty alone can overemphasize prompts that are either too easy or too difficult. To capture learnability more directly, we use the posterior expectation of the Bernoulli variance:
\[
q_{k,i} := \mathbb E\left[p_i(1-p_i)\mid \mathcal H_i^k\right].
\]
Under the Beta posterior, this informativeness score has the closed-form expression
\[
q_{k,i}=\frac{a_{k,i}b_{k,i}}{(a_{k,i}+b_{k,i})(a_{k,i}+b_{k,i}+1)}.
\]
The term \(p_i(1-p_i)\) is maximized near \(p_i=1/2\) and vanishes in the degenerate regimes \(p_i\in\{0,1\}\). Thus, \(q_{k,i}\) measures the extent to which the outcome of prompt \(i\) remains sensitive to additional rollouts. A larger value of \(q_{k,i}\) indicates that prompt \(i\) is still learnable and that further budget spent on this prompt is likely to produce useful training signals.

In the regret analysis, we fix \(q_i\) and analyze the corresponding fixed utility. In the implementation, this fixed score is replaced by the epoch-wise plug-in posterior estimate \(q_{k,i}\). For the fixed-utility formulation, define
\[
U_i(n):=U(q_i,n)=1-\exp(-\eta q_i n)=1-\exp(-c_i n),
\quad
c_i:=\eta q_i,
\]
where \(n\) is the cumulative number of rollouts assigned to prompt \(i\), and \(\eta>0\) is a temperature parameter. Larger values of \(\eta\) make the utility saturate more quickly, placing greater emphasis on early allocations, while smaller values lead to a more gradual increase and weaker diminishing returns.

This utility has two useful properties. First, it is increasing in both the cumulative allocation \(n\) and the informativeness score \(q_i\). Second, it is concave in \(n\), which captures diminishing returns: early rollouts are typically more informative, while the marginal value of additional rollouts decreases as more samples are collected. Consequently, this utility favors prompts that are both informative and under-explored, while preventing excessive allocation to a single prompt.

Substituting this utility into the fixed-budget allocation problem gives
\begin{equation}
\label{eq:allocation_final}
\begin{aligned}
\OPT(\mathcal X) := \max_{\{N_{i,k}\}} \quad
& \sum_{i=1}^M \left[ 1-\exp\left(-\eta q_i\sum_{k=1}^K N_{i,k}\right)\right] \\
\text{s.t.}\quad
& \sum_{i=1}^M\sum_{k=1}^K N_{i,k}\le B,\\
& 0\le N_{i,k}\le \Nmax,
\quad i\in[M],\ k\in[K].
\end{aligned}
\end{equation}
This formulation yields a simple Bayesian mechanism for adaptive rollout allocation: the implementation uses historical rollouts to update posterior statistics \(q_{k,i}\), and then allocates future rollout budget toward prompts whose outcomes are expected to remain informative.

\subsection{Theoretical Motivation}
\label{subsec:theoretical-motivation}

Our notion of informativeness is motivated by recent analyses of prompt-level learnability in reinforcement learning with verifiable rewards. Under binary rewards, the policy gap to the entropy-regularized optimum is governed by the variance of the Bernoulli reward distribution \citep{bae2026online}. Let \(p_i\) denote the pass probability for prompt \(x_i\). We define
\begin{equation}
I_i := p_i(1-p_i),
\end{equation}
which attains its maximum at \(p_i=1/2\) and vanishes when \(p_i\in\{0,1\}\). This design is also consistent with recent findings in difficulty-aware RL for LLMs \citep{shrivastava2025sample,hu2025vade}: prompts that are either too easy or too difficult tend to provide limited learning signal, since the model almost always succeeds or fails on them, whereas medium-difficulty prompts induce higher outcome variance and thus yield more informative updates. In our setting, \(I_i\) captures this principle directly: easy and hard prompts, corresponding to \(p_i\) close to \(1\) or \(0\), receive fewer rollouts, while prompts of intermediate difficulty, where \(p_i\approx 1/2\), are assigned more rollouts.

This observation provides a natural justification for using reward variation as a measure of informativeness. However, in our rollout allocation setting, the latent pass probability \(p_i\) is unknown and must be estimated online from historical rollouts. Rather than using a point estimate of \(p_i\), we maintain the posterior distribution
\[
p_i\mid \mathcal H_i^k \sim \mathrm{Beta}(a_{k,i},b_{k,i}),
\]
which captures both the current estimate of the prompt difficulty and the remaining uncertainty in that estimate. We then define the posterior informativeness as the posterior expectation of the ideal learnability:
\[
q_{k,i} := \mathbb E[I_i\mid \mathcal H_i^k] = \mathbb E[p_i(1-p_i)\mid \mathcal H_i^k].
\]
Under the Beta posterior, this yields
\[
q_{k,i} = \frac{a_{k,i}b_{k,i}} {(a_{k,i}+b_{k,i})(a_{k,i}+b_{k,i}+1)}.
\]

While prior work uses this variance-based learnability mainly to filter prompts into retained or discarded subsets, we use it as a continuous Bayesian score for optimizing rollout allocation under a fixed budget. Thus, \(q_{k,i}\) can be interpreted as a Bayesian proxy for the ideal prompt-level learnability. Compared with directly filtering prompts based on empirical pass rates, this posterior formulation is better suited for online rollout allocation: it accounts for uncertainty induced by limited historical samples and provides a smooth, continuously updated estimate of how informative future rollouts from each prompt are expected to be.

This motivates the plug-in utility used by the implementation,
\[
U(q_{k,i},n) = 1-\exp(-\eta q_{k,i}n),
\]
where larger \(q_{k,i}\) increases the marginal value of additional rollouts, while concavity in \(n\) imposes diminishing returns and prevents excessive allocation to a single prompt. The theorem below analyzes the fixed-utility version obtained by holding \(q_i\) fixed; the posterior score \(q_{k,i}\) is used as a practical plug-in update in CERO.

\section{Online Rollout Allocation Algorithm}

We now present \textsc{CERO}, an online primal-dual rollout allocation algorithm for LLM reinforcement learning under a fixed generation budget. 
The formal pseudo-code is provided in \Cref{alg:rollout-allocation}.

\paragraph{Online rollout allocation.}
Given the current dual variables $(\bm{\theta}^k,\mu^k)$, \textsc{CERO} determines the number of rollouts to collect for each prompt by solving the prompt-wise allocation subproblem
\[
N_{i,k}
=
\arg\max_{0\le n\le \Nmax}
(\theta_i^k-\mu^k)n,
\label{eq:cero-allocation-rule}
\]
where $\Nmax$ is the per-prompt rollout cap.  
Since the objective is linear in $n$, the solution admits the threshold form
\[
N_{i,k}
=
\begin{cases}
\Nmax, & \theta_i^k>\mu^k,\\
0, & \theta_i^k\le \mu^k.
\end{cases}
\]
In implementation, we additionally clip the allocation by the remaining rollout budget to avoid overspending. The regret analysis below is stated for the primal--dual allocation rule with the box constraint \(0\le n\le\Nmax\), and controls budget violation through the global dual variable.
Thus, a prompt receives rollouts only when its estimated marginal value exceeds the current global budget price. 
This rule provides an online scheduler for LLM rollout generation: prompts that are currently more informative receive more samples, while low-value prompts are skipped to preserve budget for future epochs.

\paragraph{Prompt-level feedback and dual update.}
After selecting $N_{i,k}$, the current policy samples responses from prompt $x_i$,
\[
y_{i,k}^{(1)},\ldots,y_{i,k}^{(N_{i,k})}
\sim
\pi_\phi(\cdot\mid x_i),
\]
and the generated responses are evaluated by the reward model or task-specific reward function. 
The resulting prompt--response--reward tuples are added to the rollout buffer $\mathcal D$ and are used to update the informativeness estimate of prompt $x_i$. 
Specifically, \textsc{CERO} updates the Beta posterior associated with prompt $x_i$, computes the next informativeness score $q_{k+1,i}$, and recomputes the corresponding utility term.

Using this updated prompt-level feedback, \textsc{CERO} performs a projected online gradient step on the prompt-level dual variable:
\[
g_{\theta_i}^k
=
N_{i,k}
-
\frac{1}{K}\nabla U_i^*(\theta_i^k),
\qquad
\theta_i^{k+1}
=
\left[
\theta_i^k-\eta_\theta g_{\theta_i}^k
\right]_+.
\label{eq:cero-theta-update}
\]
This update follows from the $\theta_i$-subgradient of the per-epoch dual objective. 
Intuitively, it adjusts the future sampling priority of prompt $x_i$ by balancing the number of rollouts already allocated to the prompt against the utility implied by its current informativeness estimate.

\paragraph{Cross-epoch budget control.}
To regulate rollout usage over the entire training horizon, \textsc{CERO} updates the global multiplier according to the discrepancy between the current epoch's rollout consumption and the remaining budget $B_k$:
\[
g_\mu^k
=
\frac{B_k}{K}
-
\sum_{i=1}^M N_{i,k},
\qquad
\mu^{k+1}
=
\left[
\mu^k-\eta_\mu g_\mu^k
\right]_+.
\label{eq:cero-mu-update}
\]
When the algorithm consumes rollouts faster than the target budget rate, $\mu^k$ increases, making future allocation decisions more selective. 
Conversely, when rollout consumption is below the target rate, $\mu^k$ decreases, allowing more prompts to become eligible for sampling in later epochs. 
This global update couples all prompt-level allocation decisions through the shared generation budget, while avoiding a rigid per-epoch rollout quota.

\paragraph{Integration with existing LLM RL pipelines.}
\textsc{CERO} can be used as a drop-in online rollout scheduler for existing LLM reinforcement-learning pipelines.
In our implementation, we instantiate \textsc{CERO} on top of the widely used GRPO framework.
Whenever the rollout buffer satisfies $|\mathcal D|\ge B_{\mathrm{batch}}$, the policy is updated using the GRPO objective and the buffer is cleared.
Importantly, \textsc{CERO} leaves the underlying policy-gradient objective unchanged; it only modifies the data-collection process.
By adapting rollout allocation to prompt-level feedback and global budget pressure, \textsc{CERO} improves rollout efficiency under a fixed sampling budget.

\begin{algorithm}[t]
\caption{Cross-Epoch Online Adaptive Rollout Optimization (CERO)}
\label{alg:rollout-allocation}
\begin{algorithmic}[1]
\State \textbf{Input:} Number of prompts $M$, training epochs $K$, average rollout $N$, maximum rollouts $\Nmax$, utility functions $U_i(\cdot)$, policy $\pi_\theta$, batch size $B_{\text{batch}}$
\State Initialize dual variables $\boldsymbol{\theta}^1$, $\mu^1$, total rollout budget $B_0 = KMN$, collected batch $\mathcal{D} = \emptyset$
\For{$k = 1$ to $K$}
    \For{each prompt $x_i$}
        \State Compute the step allocation: $N_{i,k} = \arg\max_{0 \le N_{i,k} \le \min(\Nmax, B - B_{\text{used}})} (\theta_i^k - \mu^k) \cdot N_{i,k}$.
        
        \State Allocate $N_{i,k}$ rollouts: $y_i^{(1)}, \dots, y_i^{(N_{i,k})} \sim \pi_\theta(\cdot \mid x_i)$.
        \State Add collected rollouts to batch: $\mathcal{D} \gets \mathcal{D} \cup \{y_i^{(1)}, \dots, y_i^{(N_{i,k})}\}$
        \State Update remaining budget: $B_k \gets B_k - N_{i,k}$
        \State Update dual variables via online gradient descent:
        \[
        g_{\theta_i}^k = {N}_{i,k} - \frac{1}{K} \nabla U_i^*(\theta_i^k), \quad\theta_i^{k+1} = \left[\theta_i^k - \eta_\theta g_{\theta_i}^k \right]_+
        \]
        
        \If{$|\mathcal{D}| \ge B_{\text{batch}}$}
            \State \textbf{Policy update}
        \EndIf
    \EndFor
    
    \If{$B < 0$} 
        \State \textbf{Terminate}
    \EndIf
    
    \State Update global dual variable:
    \[
    g_\mu^k = \frac{B_k}{K} - \sum_{i=1}^M {N}_{i,k}, \quad\mu^{k+1} = \left[\mu^k - \eta_\mu g_\mu^k \right]_+, \quad B_{k+1} = B_k
    \]
\EndFor
\end{algorithmic}
\end{algorithm}

\subsection{Theoretical Guarantee}

Before stating the regret guarantee, we first define the offline benchmark against which the online rollout allocation algorithm is compared. 
For the theoretical analysis, we consider the setting where each prompt $x_i$ has a fixed informativeness score $q_i$. 
To simplify notation, we suppress the dependence on $q_i$ and the utility hyperparameter $\eta$.

Given a total rollout budget $B$, let $\OPT$ denote the optimal value of the offline allocation problem in Eq.~\eqref{eq:allocation_final}. 
For \textsc{CERO} algorithm, we define its realized utility as
\[
    \ALG
    :=
    \sum_{i=1}^M
    U_i\left(\sum_{k=1}^K N_{i,k}\right),
\]
where $N_{i,k}$ is the number of rollouts allocated to prompt $x_i$ at epoch $k$. 
The regret of \textsc{CERO} is defined as following:
\[
    \Reg := \OPT - \ALG
\]

\begin{theorem}[main result]
\label{thm:rollout-allocation-regret}
In the LLM RL post-training setting, suppose that each training prompt appears only once in each training epoch, and consider the CERO algorithm in \Cref{alg:rollout-allocation}. Suppose the prompt-level Fenchel dual variables and the global rollout-budget dual variable are projected onto compact domains,
\[
    \bm\theta^k\in\Theta_\epsilon:=\prod_{i=1}^M[\epsilon,c_i],
    \quad
    \mu^k\in\mathcal M:=[0,\bar\mu],
\]
where \(c_i=\eta q_i\), \(\epsilon>0\), and \(\bar\mu<\infty\). If the prompt-level and budget-level OCO updates incur regrets \(\Reg_\theta(K)\) and \(\Reg_\mu(K)\), then
\[
    \Reg \le \Reg_\theta(K)+\Reg_\mu(K).
\]
Consequently, if projected online gradient ascent is used on \(\Theta_\epsilon\) and \(\mathcal M\) with uniformly bounded supergradients, then
\[
    \Reg=O(\sqrt K).
\]
\end{theorem}

\section{Numerical Experiments}

\subsection{Experimental Setup}

We integrate our proposed \textbf{CERO} with GRPO and compare it with the vanilla GRPO algorithm.
We evaluate their performance on four open-weight language models:
\textbf{DeepSeek-R1-Distill-Qwen-1.5B} \citep{guo2025deepseek} (abbreviated as R1-Distill-1.5B),
\textbf{Qwen3-4B-Base} \citep{yang2025qwen3},
\textbf{Qwen3-4B-Instruct} \citep{yang2025qwen3},
and \textbf{Qwen2.5-Math-7B} \citep{yang2024qwen2}.
We post-training these LLMs via verl \citep{sheng2025hybridflow} framework on DAPO-Math-17K dataset \citep{yu2025dapo}, which consists of $17917$ training prompts with ground truth answers for verification.  We initiate the Beta posterior for training prompts as $\mathrm{Beta}(1,1)$. More training details are deferred in Appendix \ref{app:exp-details}.

We evaluate mathematical reasoning on four competition-style benchmarks:
\textbf{AIME24}, \textbf{AIME25}, \textbf{AIME26}, and \textbf{AMC23}.
The AIME sets test challenging multi-step high-school reasoning with exact-match integer answers, while AMC23 provides broader coverage with more varied difficulty and problem formats. Together, these benchmarks assess whether improvements generalize across contest years and beyond AIME-style problems.

\subsection{Main results and analysis}

Table~\ref{tab:main_results} compares CERO with vanilla GRPO across four representative models and four mathematical reasoning benchmarks. CERO consistently outperforms GRPO for every evaluated model--benchmark pair, indicating that its gains are not tied to a particular architecture, model scale, or post-training setting. Averaged across benchmarks, CERO improves over GRPO by $+4.84$ points on R1-Distill-1.5B, $+6.04$ points on Qwen3-4B-Base, $+3.28$ points on Qwen3-4B-Instruct, and $+1.53$ points on Qwen2.5-Math-7B. Such results demonstrate the effectiveness of our proposed method. In addition, CERO introduces negligible computational overhead: in our experiments, each allocation update takes less than one second.



\begin{table*}[t]
\centering
\small
\setlength{\tabcolsep}{7pt}
\renewcommand{\arraystretch}{1.15}

\caption{Evaluation performance comparison across different models and benchmarks.}
\label{tab:main_results}

\begin{tabular}{lccccccc}
\toprule
& \rotatebox{55}{AIME24}
& \rotatebox{55}{AIME25}
& \rotatebox{55}{AIME26}
& \rotatebox{55}{AMC}
& \rotatebox{55}{Avg} \\
\midrule

R1-Distill-1.5B  + GRPO
& 0.1041 & 0.1416 & 0.0875 & 0.5468 & 0.22\\
R1-Distill-1.5B  + KnapsackRL
&  &  &  &  & \\
R1-Distill-1.5B + CERO
& \textbf{0.152} & \textbf{0.1666} & \textbf{0.1208} & \textbf{0.6343} & \textbf{0.2684} \\
\midrule

Qwen3-4B-base + GRPO
& 0.15 & 0.1312 & 0.1104 & 0.5046 & 0.2241 \\
Qwen3-4B-base + KnapsackRL
& &  & &  & \\
Qwen3-4B-base + CERO
& \textbf{0.1875} & \textbf{0.1833} & \textbf{0.1812} & \textbf{0.5875} & \textbf{0.2845} \\
\midrule

Qwen3-4B-Instruct + GRPO
& 0.3208 & 0.3291 & 0.3687 & 0.8546 & 0.4683 \\
Qwen3-4B-Instruct + KnapsackRL
&  &  &  &  &  \\
Qwen3-4B-Instruct + CERO

& \textbf{0.4} & \textbf{0.3625} & \textbf{0.3812} & \textbf{0.8609} & \textbf{0.5011} \\
\midrule

Qwen2.5-Math-7B + GRPO
& 0.2396 & 0.1541 & \textbf{0.1250} & 0.6531 & 0.2913 \\
Qwen2.5-Math-7B + KnapsackRL
&  & & & & \\
Qwen2.5-Math-7B + CERO
& \textbf{0.2625} & \textbf{0.1750} & 0.1187 & \textbf{0.6703} & \textbf{0.3066} \\
\bottomrule

\end{tabular}
\end{table*}

We next analyze the source of CERO's gains from the perspective of training signal availability. Following Knapsack-RL~\citep{li2025knapsack}, we define an \textbf{effective prompt} as a prompt whose rollout group contains both successful and failed responses. Such prompts have non-zero within-group reward variance and therefore induce non-zero GRPO gradients. We use the \emph{effective prompt ratio} to measure the fraction of prompts in a batch that provide useful learning signals, and report this metric in \Cref{fig:effective-ratio}.

Across all LLM post-training settings, CERO consistently achieves a higher effective prompt ratio than GRPO, which helps explain its stronger downstream performance. In particular, \Cref{fig:15B_effective_prompt_ratio} shows that the effective prompt ratio of GRPO drops sharply after a certain number of training steps. This trend suggests that, by that stage, DeepSeek-R1-Distill-Qwen-1.5B has already learned many of the sampled prompts; as a result, later rollout groups often contain only successful or failure responses, yielding ineffective prompts with zero reward variance. In contrast, CERO adaptively allocates more rollout budget to informative prompts and stops before the effective prompt ratio collapses, thereby maintaining a higher proportion of useful training signals throughout training.

\begin{figure}[t]
    \centering
    \begin{subfigure}[t]{0.45\linewidth}
        \centering
        \includegraphics[width=\linewidth]{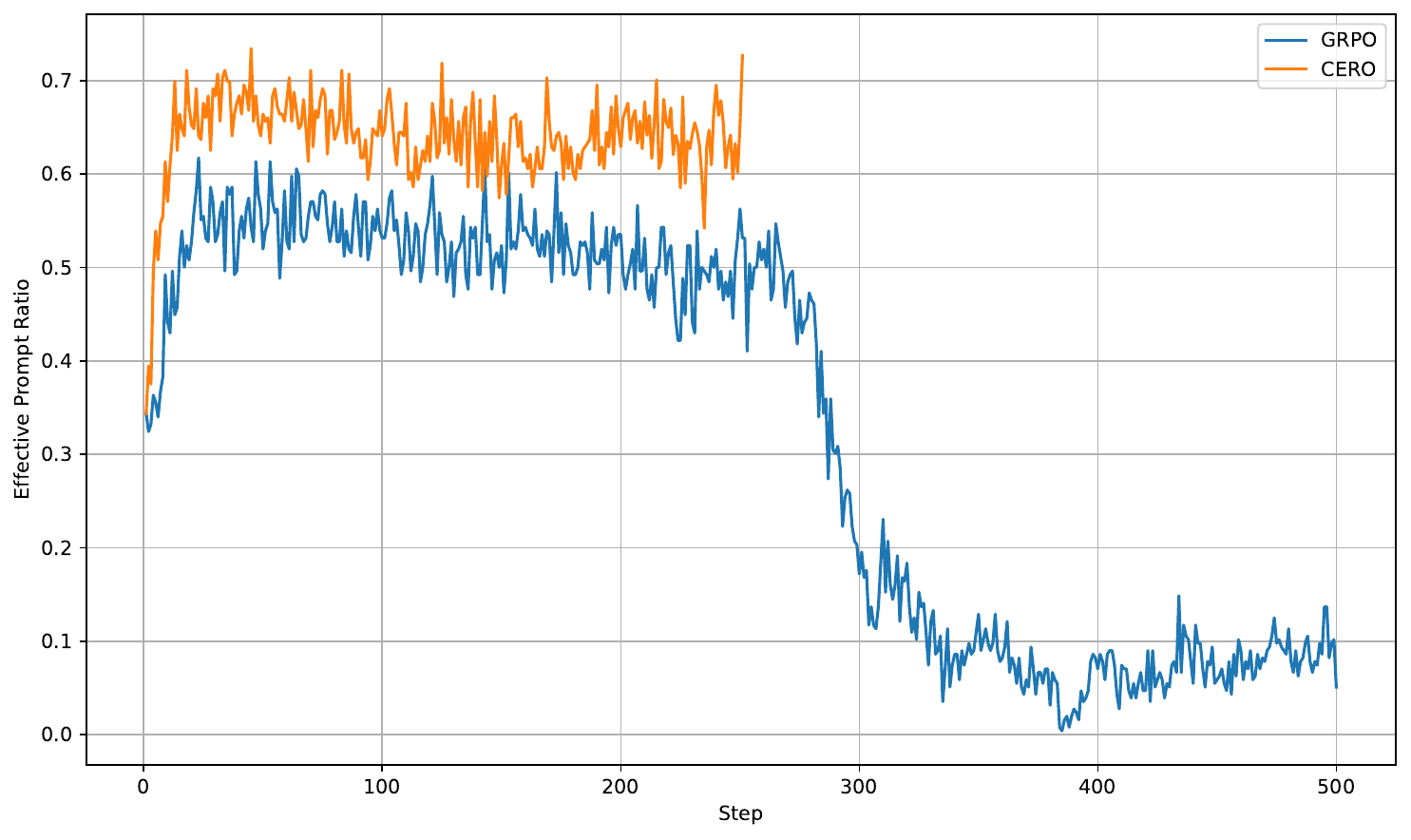}
        \caption{DeepSeek-R1-Distill-Qwen-1.5B}
        \label{fig:15B_effective_prompt_ratio}
    \end{subfigure}
    \hfill
    \begin{subfigure}[t]{0.45\linewidth}
        \centering
        \includegraphics[width=\linewidth]{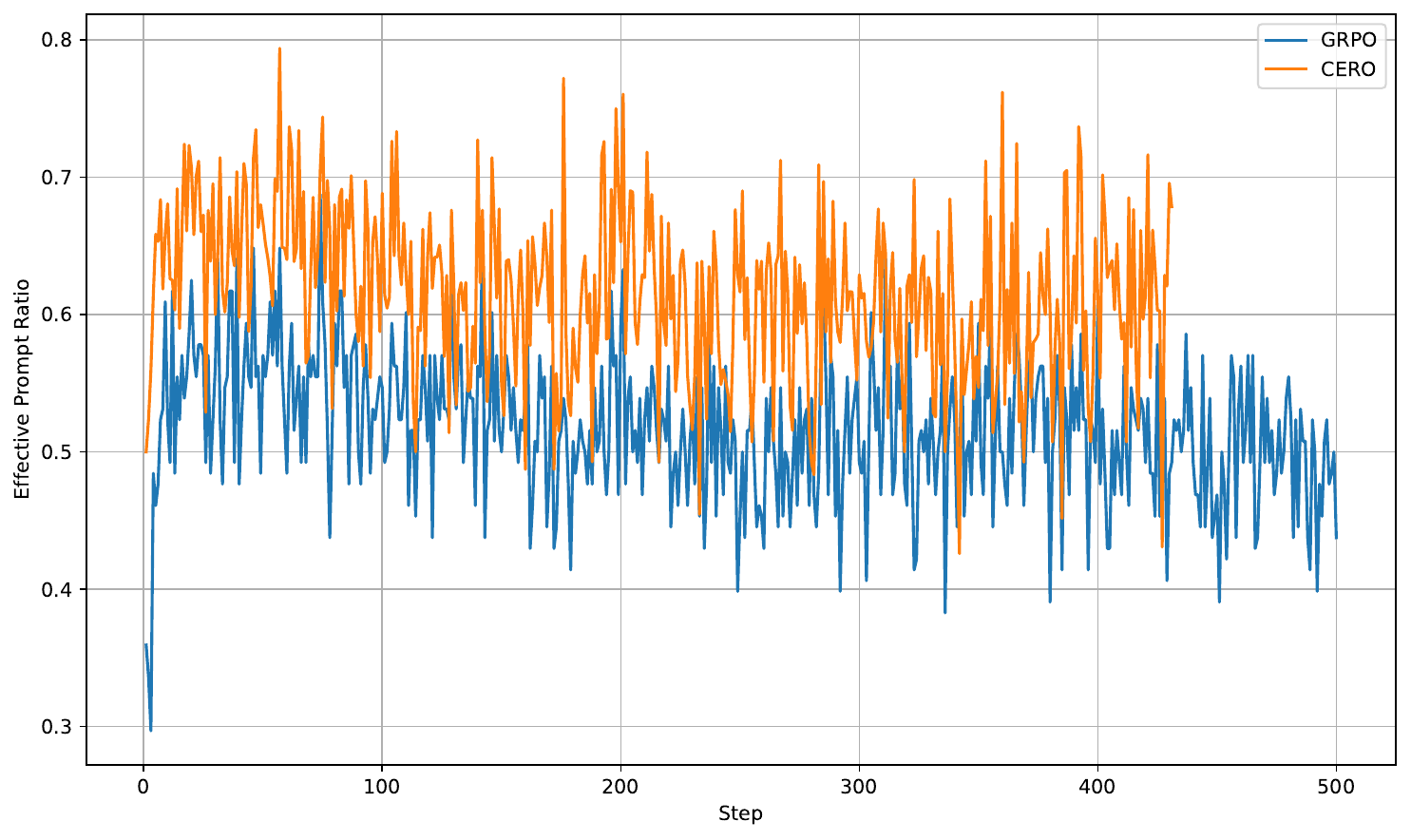}
        \caption{Qwen3-4B-base}
        \label{fig:4B_base_effective_prompt_ratio}
    \end{subfigure}
    \hfill
    \begin{subfigure}[t]{0.45\linewidth}
        \centering
        \includegraphics[width=\linewidth]{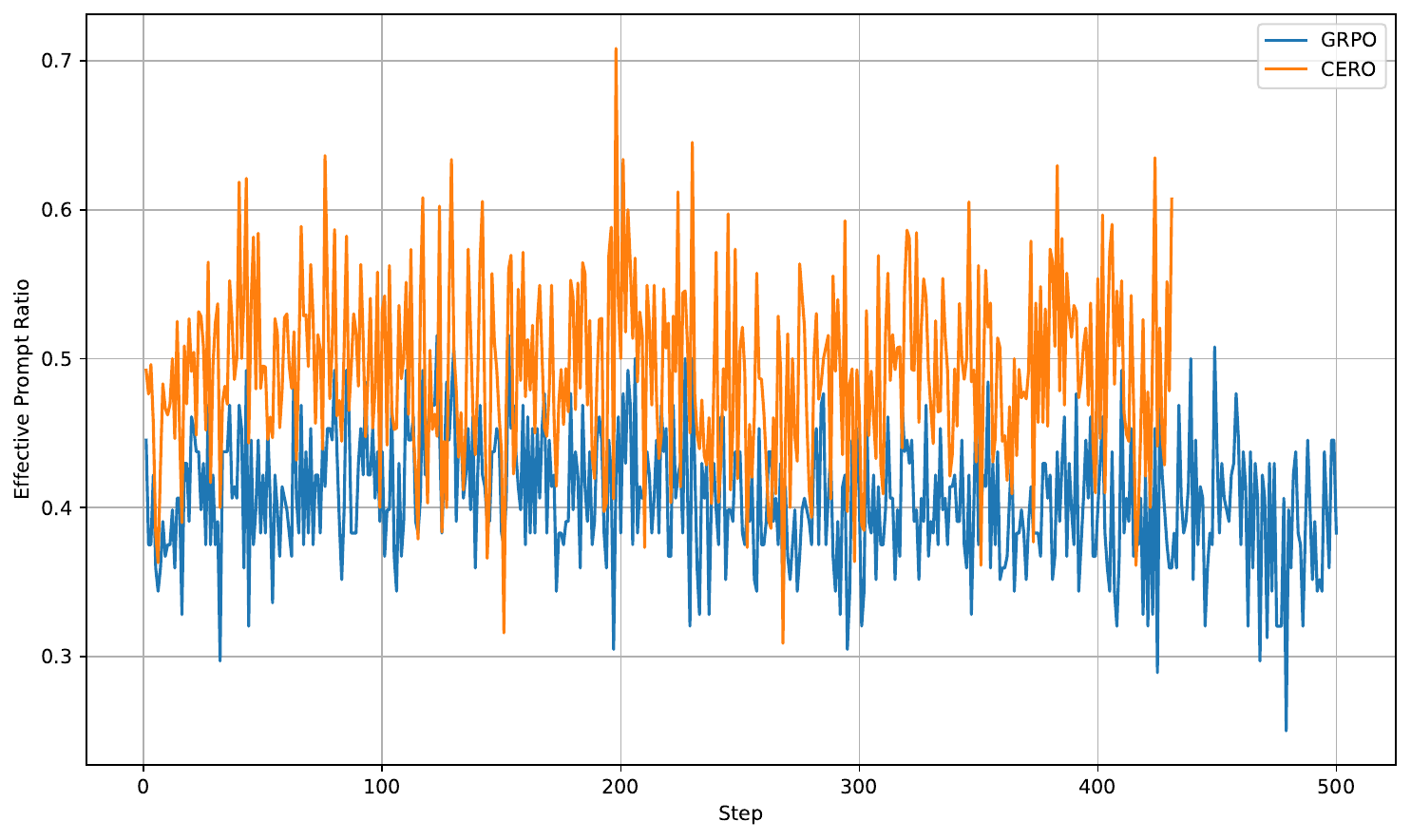}
        \caption{Qwen3-4B-Instruct}
        \label{fig:4B_Instruct_effective_prompt_ratio}
    \end{subfigure}
    \hfill
    \begin{subfigure}[t]{0.45\linewidth}
        \centering
        \includegraphics[width=\linewidth]{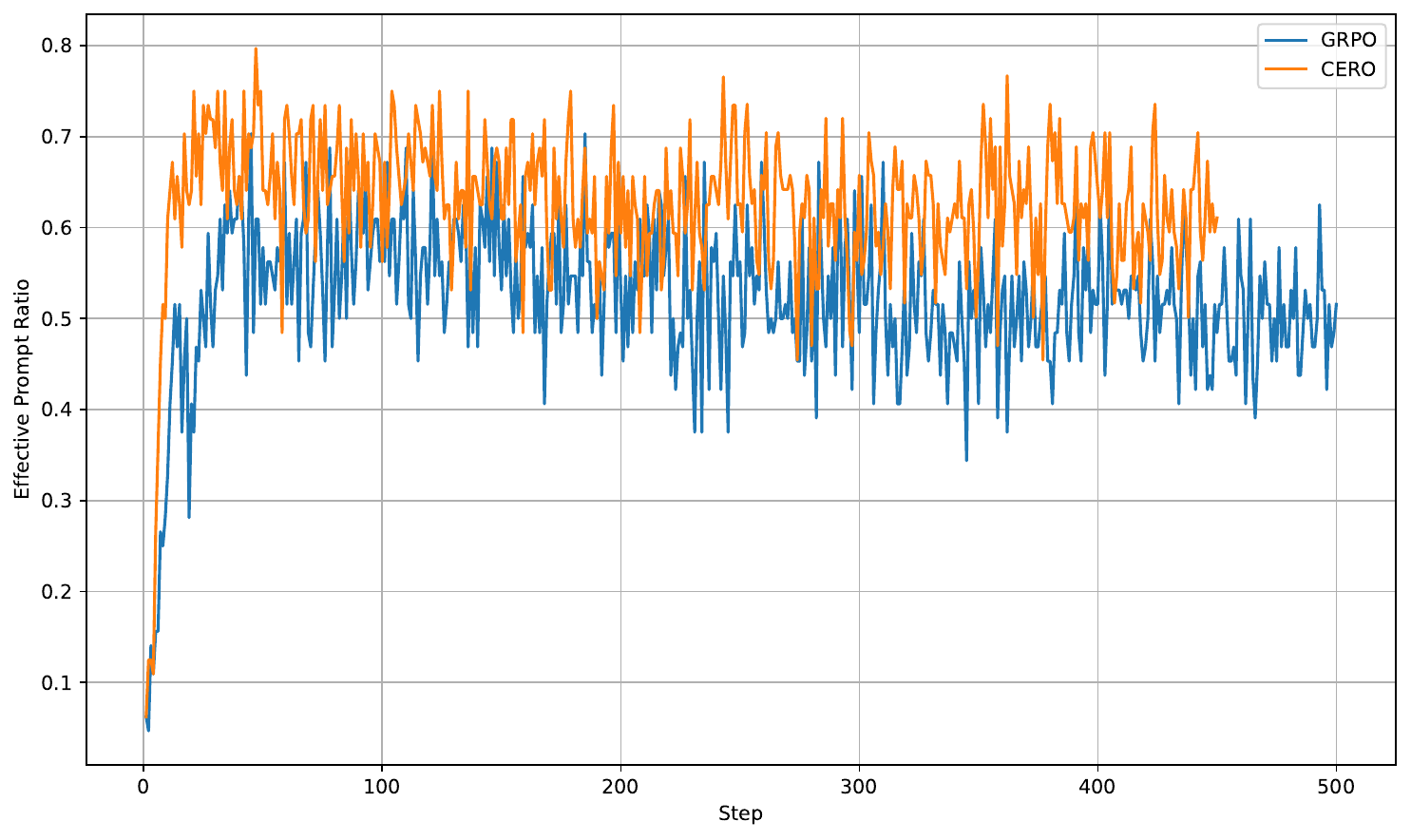}
        \caption{Qwen2.5-Math}
        \label{fig:7B_effective_prompt_ratio}
    \end{subfigure}
    \caption{Effective prompt ratio comparison between CERO and GRPO.}
    \label{fig:effective-ratio}
\end{figure}

\section{Conclusion}

We introduced \textsc{CERO}, a cross-epoch adaptive rollout allocation framework for LLM RL post-training under a fixed global budget. By modeling prompt informativeness with a Beta posterior, using a concave saturating utility to capture diminishing returns, and deriving a Fenchel-dual online optimization formulation, \textsc{CERO} adaptively prioritizes prompts that provide higher expected learning value. We established an $O(\sqrt{K})$ regret guarantee and empirically showed consistent improvements over GRPO across multiple models and mathematical reasoning benchmarks. These results suggest that adaptive rollout budgeting is an effective and principled direction for improving the sample efficiency of RL post-training.

\paragraph{Limitations.}
Our theoretical regret guarantee is established for fixed prompt-level utilities, whereas the implemented version uses epoch-wise plug-in posterior informativeness scores. Thus, the current analysis does not fully capture the nonstationarity induced by policy updates and posterior refreshes. Empirically, we evaluate CERO on mathematical-reasoning post-training with four open-weight backbones and four benchmarks; broader validation on other domains, reward models, and larger-scale training regimes remains important.

\bibliographystyle{abbrvnat}
\bibliography{bibliography}

\clearpage

\begin{APPENDICES}
\section{Proof of Theorem~\ref{thm:rollout-allocation-regret}}
\label{app:proof-rollout-allocation}

\paragraph{Step 1: Fenchel representation of \(U_i\).}
Recall that the prompt-level utility is
\[
U_i(n)=1-\exp(-c_i n),
\quad c_i:=\eta q_i ,
\]
where \(q_i\) is the informativeness score of prompt \(i\) and \(\eta>0\) is the temperature parameter. Throughout the proof, we treat \(q_i\) and hence \(c_i\) as fixed. The posterior update used in the implementation can be viewed as a plug-in procedure for updating these utilities over time.

For \(n\ge 0\), we have
\[
U_i'(n)=c_i \exp(-c_i n).
\]
Hence
\[
0 < U_i'(n) \le c_i .
\]
Therefore, the natural domain of the Fenchel dual variable is
\[
\theta_i \in [0,c_i].
\]

For a concave utility function, define its concave Fenchel conjugate by
\[
U_i^*(\theta_i):=\inf_{n\ge 0}\left\{\theta_i n - U_i(n)\right\}.
\]
For \(U_i(n)=1-\exp(-c_i n)\), this gives
\[
U_i^*(\theta_i)=\inf_{n\ge 0}\left\{\theta_i n -1+\exp(-c_i n)\right\}.
\]
For \(0<\theta_i\le c_i\), the first-order condition is
\[
\theta_i - c_i\exp(-c_i n)=0.
\]
Thus
\[
\exp(-c_i n)=\frac{\theta_i}{c_i},
\quad
n=\frac{1}{c_i}\log\frac{c_i}{\theta_i}.
\]
Substituting this optimizer back into the conjugate yields
\[
\begin{aligned}
U_i^*(\theta_i)
&=\theta_i \cdot \frac{1}{c_i}\log\frac{c_i}{\theta_i}-1+\frac{\theta_i}{c_i}  \\
&=\frac{\theta_i}{c_i}\log\frac{c_i}{\theta_i}+\frac{\theta_i}{c_i}-1 ,
\quad 0<\theta_i\le c_i .
\end{aligned}
\]
The value at \(\theta_i=0\) is understood by continuity:
\[
U_i^*(0)=-1.
\]
Therefore,
\[
U_i^*(\theta_i)=\frac{\theta_i}{c_i}\log\frac{c_i}{\theta_i}+\frac{\theta_i}{c_i}-1,
\quad
\theta_i\in[0,c_i],
\]
with the convention \(0\log(c_i/0)=0\).

By concave Fenchel duality, for every \(n\ge 0\),
\[
U_i(n)=\min_{\theta_i\in[0,c_i]}\left\{\theta_i n - U_i^*(\theta_i)\right\}.
\]
Equivalently,
\[
-U_i(n)=\max_{\theta_i\in[0,c_i]}\left\{U_i^*(\theta_i)-\theta_i n\right\}.
\]

For technical convenience, the algorithm projects onto the compact truncated domain
\[
\Theta_i=[\epsilon,c_i],
\]
where \(\epsilon>0\) is arbitrarily small. The approximation error from replacing \([0,c_i]\) by \([\epsilon,c_i]\) can be absorbed into the dual-regret term.

\paragraph{Step 2: OCO control of the utility term.}
Let
\[
n_i:=\sum_{k=1}^K N_{i,k}
\]
denote the total number of rollouts allocated to prompt \(i\) by the online algorithm. The realized utility of the algorithm is
\[
\mathrm{ALG}=\sum_{i=1}^M U_i(n_i)=\sum_{i=1}^MU_i\!\left(\sum_{k=1}^K N_{i,k}\right).
\]
Using the concave Fenchel representation from Step~1, we have, for each \(i\),
\[
-U_i(n_i)=\max_{\theta_i\in[0,c_i]}\left\{U_i^*(\theta_i)-\theta_i n_i\right\}.
\]
Therefore,
\[
\begin{aligned}
-\mathrm{ALG}
&=\sum_{i=1}^M\left[-U_i\!\left(\sum_{k=1}^K N_{i,k}\right)\right]  \\
&=\max_{\theta\in\Theta}\sum_{i=1}^M\left[U_i^*(\theta_i)-\theta_i\sum_{k=1}^K N_{i,k}\right],
\end{aligned}
\]
where
\[
\Theta:=\prod_{i=1}^M[0,c_i].
\]
Rewriting the right-hand side as a sum over epochs gives
\[
\begin{aligned}
\sum_{i=1}^M\left[U_i^*(\theta_i)-\theta_i\sum_{k=1}^K N_{i,k}\right]
&=\sum_{k=1}^K\sum_{i=1}^M\left[\frac{1}{K}U_i^*(\theta_i)-\theta_i N_{i,k}\right].
\end{aligned}
\]
Define the prompt-level dual reward function
\[
\psi_k(\theta):=\sum_{i=1}^M\left[\frac{1}{K}U_i^*(\theta_i)-\theta_i N_{i,k}\right].
\]
Then
\[
-\mathrm{ALG}=\max_{\theta\in\Theta}\sum_{k=1}^K \psi_k(\theta).
\]

For technical convenience, the algorithm projects onto the truncated compactdomain
\[
\Theta_\epsilon:=\prod_{i=1}^M[\epsilon,c_i],
\]
where \(\epsilon>0\) is arbitrarily small. The error caused by replacing\(\Theta\) with \(\Theta_\epsilon\) can be absorbed into the prompt-level dualregret term.

Assume that the prompt-level online concave maximization subroutine guarantees
\[
\max_{\theta\in\Theta_\epsilon}\sum_{k=1}^K \psi_k(\theta)-\sum_{k=1}^K \psi_k(\theta^k)\le\mathrm{Reg}_\theta(K).
\]
Then
\begin{equation}
-\ALG
\le
\sum_{k=1}^K
\sum_{i=1}^M
\left[
\frac1K U_i^*(\theta_i^k)
-
\theta_i^k N_{i,k}
\right]
+
\Reg_\theta(K).
\label{eq:utility-oco-control}
\end{equation}

\paragraph{Step 3: OCO control of the budget violation.}
Let $b:=\frac{B}{K}$ be the average rollout budget per epoch. The total budget violation of the online allocation is
\[
\left[\sum_{k=1}^K\sum_{i=1}^M N_{i,k}-B\right]_+.
\]
Since \(B=Kb\), we can rewrite the violation as
\[
\left[\sum_{k=1}^K\left(\sum_{i=1}^M N_{i,k}-b\right)\right]_+.
\]

Let \(\mathcal M:=[0,\bar\mu]\) be the compact domain for the global budget dual variable, where \(\bar\mu\ge 1\). For any scalar \(z\), we have
\[
[z]_+\le\max_{\mu\in\mathcal M} \mu z .
\]
Applying this inequality with
\[
z=\sum_{k=1}^K\left(\sum_{i=1}^M N_{i,k}-b\right),
\]
we obtain
\[
\begin{aligned}
\left[\sum_{k=1}^K\sum_{i=1}^M N_{i,k}-B\right]_+
&\le\max_{\mu\in\mathcal M}\sum_{k=1}^K\mu\left(\sum_{i=1}^M N_{i,k}-b\right).
\end{aligned}
\]

Define the budget dual reward function
\[
q_k(\mu):=\mu\left(\sum_{i=1}^M N_{i,k}-b\right).
\]
Assume that the global budget online concave maximization subroutine guarantees
\[
\max_{\mu\in\mathcal M}\sum_{k=1}^K q_k(\mu)-\sum_{k=1}^K q_k(\mu^k)\le\mathrm{Reg}_\mu(K).
\]
Then
\begin{equation}
\label{eq:budget-oco-control}
\begin{aligned}
\left[\sum_{k=1}^K\sum_{i=1}^M N_{i,k}-B\right]_+
&\le\sum_{k=1}^K q_k(\mu^k)+\mathrm{Reg}_\mu(K) \\
&=\sum_{k=1}^K\mu^k\left(\sum_{i=1}^M N_{i,k}-b\right)+\mathrm{Reg}_\mu(K).
\end{aligned}
\end{equation}

Adding \eqref{eq:utility-oco-control} and \eqref{eq:budget-oco-control}, we get
\begin{align}
&-\ALG+\left[\sum_{k=1}^K\sum_{i=1}^M N_{i,k}-B\right]_+\nonumber\\
&\le\sum_{k=1}^K\left\{\sum_{i=1}^M\left[\frac1K U_i^*(\theta_i^k)-\theta_i^k N_{i,k}\right]+\mu^k\left(\sum_{i=1}^M N_{i,k}-b\right)\right\}+\Reg_\theta(K)+\Reg_\mu(K).
\label{eq:combined-oco-control}
\end{align}

\paragraph{Step 4: Comparison with an evenly spread offline optimum.}
Let \(N_{i,k}^\star\) be an optimal offline allocation for the benchmark, and define the cumulative offline allocation for prompt \(i\) by
\[
n_i^\star:=\sum_{k=1}^K N_{i,k}^\star .
\]
Consider the evenly spread version of this offline allocation:
\[
\bar N_{i,k}^\star:=\frac{n_i^\star}{K},
\quad i\in[M],\ k\in[K].
\]
Since \(0\le N_{i,k}^\star\le N_{\max}\), we have
\[
0\le n_i^\star\le K N_{\max},
\]
and therefore
\[
0\le \bar N_{i,k}^\star\le N_{\max}.
\]
Thus \(\bar N_{i,k}^\star\) is feasible for the per-round allocation subproblem in the continuous relaxation.

At epoch \(k\), the allocation rule chooses
\[
N_{i,k}\in\arg\max_{0\le n\le N_{\max}}(\theta_i^k-\mu^k)n .
\]
Therefore, for every \(i\) and \(k\),
\[
(\theta_i^k-\mu^k)N_{i,k}\ge(\theta_i^k-\mu^k)\bar N_{i,k}^\star .
\]
Equivalently,
\[
-(\theta_i^k-\mu^k)N_{i,k}\le-(\theta_i^k-\mu^k)\bar N_{i,k}^\star .
\]
Adding \(\frac{1}{K}U_i^*(\theta_i^k)\) and summing over \(i\), we obtain
\begin{equation}
\begin{aligned}
&\sum_{i=1}^M\left[\frac{1}{K}U_i^*(\theta_i^k)-\theta_i^k N_{i,k}\right]+\mu^k\left(\sum_{i=1}^M N_{i,k}-b\right)  \\
&\le\sum_{i=1}^M\left[\frac{1}{K}U_i^*(\theta_i^k)-\theta_i^k \bar N_{i,k}^\star\right]+\mu^k\left(\sum_{i=1}^M \bar N_{i,k}^\star-b\right).
\label{eq:primal-comparison-round}
\end{aligned}  
\end{equation}

Using \(\bar N_{i,k}^\star=n_i^\star/K\), the first term on the right-hand side becomes
\[
\begin{aligned}
\sum_{i=1}^M\left[\frac{1}{K}U_i^*(\theta_i^k)-\theta_i^k \frac{n_i^\star}{K}\right]
&=\frac{1}{K}\sum_{i=1}^M\left[U_i^*(\theta_i^k)-\theta_i^k n_i^\star\right].
\end{aligned}
\]
By the Fenchel inequality from Step~1, for any \(\theta_i^k\in[0,c_i]\),
\[
U_i(n_i^\star) \le\theta_i^k n_i^\star - U_i^*(\theta_i^k).
\]
Equivalently,
\[
U_i^*(\theta_i^k)-\theta_i^k n_i^\star \le-U_i(n_i^\star).
\]
Therefore,
\[
\frac{1}{K}\sum_{i=1}^M\left[U_i^*(\theta_i^k)-\theta_i^k n_i^\star\right]\le-\frac{1}{K}\sum_{i=1}^M U_i(n_i^\star)=-\frac{\mathrm{OPT}}{K}.
\]

It remains to control the budget term. Since the offline allocation is feasible,
\[
\sum_{i=1}^M n_i^\star=\sum_{i=1}^M\sum_{k=1}^K N_{i,k}^\star\le B .
\]
Hence
\[
\sum_{i=1}^M \bar N_{i,k}^\star=\frac{1}{K}\sum_{i=1}^M n_i^\star\le\frac{B}{K}=b.
\]
Since \(\mu^k\ge 0\), we have
\[
\mu^k\left(\sum_{i=1}^M \bar N_{i,k}^\star-b\right)\le 0.
\]
Combining the preceding inequalities gives, for every epoch \(k\),
\[
\begin{aligned}
\sum_{i=1}^M\left[\frac{1}{K}U_i^*(\theta_i^k)-\theta_i^k N_{i,k}\right]+\mu^k\left(\sum_{i=1}^M N_{i,k}-b\right)  \le-\frac{\mathrm{OPT}}{K}.
\end{aligned}
\]
Summing over $k=1,\ldots,K$, we obtain
\begin{equation}
\sum_{k=1}^K
\left\{
\sum_{i=1}^M
\left[
\frac1K U_i^*(\theta_i^k)
-
\theta_i^k N_{i,k}
\right]
+
\mu^k
\left(
\sum_{i=1}^M N_{i,k}-b
\right)
\right\}
\le
-\OPT.
\label{eq:offline-comparison}
\end{equation}

\paragraph{Step 5: Combine the inequalities}
Substituting \eqref{eq:offline-comparison} into \eqref{eq:combined-oco-control}, we get
\[
-\ALG
+
\left[
\sum_{k=1}^K\sum_{i=1}^M N_{i,k}-B
\right]_+
\le
-\OPT
+
\Reg_\theta(K)
+
\Reg_\mu(K).
\]
Rearranging,
\[
\OPT-\ALG
+
\left[
\sum_{k=1}^K\sum_{i=1}^M N_{i,k}-B
\right]_+
\le
\Reg_\theta(K)
+
\Reg_\mu(K).
\]
Since the budget violation term is nonnegative, we  can drop it from the left-hand side and obtain
\[
\OPT-\ALG
\le
\Reg_\theta(K)+\Reg_\mu(K).
\]

\begin{lemma}[OCO regret for dual ascent updates]
\label{lem:oco-dual-regret}
Let
\[
\Theta_\epsilon:=\prod_{i=1}^M[\epsilon,c_i],
\quad
\mathcal M:=[0,\bar\mu],
\]
where \(c_i=\eta q_i\), \(\epsilon>0\), and \(\bar\mu<\infty\).
Suppose that the prompt-level dual reward functions \(\psi_k:\Theta_\epsilon\to\mathbb R\) are concave and have uniformly bounded supergradients:
\[
\|\nabla \psi_k(\theta)\|_2\le G_\theta,
\quad
\forall \theta\in\Theta_\epsilon,\ k\in[K],
\]
and that the budget dual reward functions \(q_k:\mathcal M\to\mathbb R\) are concave and have uniformly bounded supergradients:
\[
|\nabla q_k(\mu)|\le G_\mu,
\quad
\forall \mu\in\mathcal M,\ k\in[K].
\]
Consider the projected online gradient ascent updates
\[
\theta^{k+1} = \Pi_{\Theta_\epsilon} \left( \theta^k+\eta_\theta\nabla\psi_k(\theta^k)\right),
\]
and
\[
\mu^{k+1}=\Pi_{\mathcal M}\left(\mu^k+\eta_\mu\nabla q_k(\mu^k)\right).
\]
Then, with
\[
\eta_\theta=\frac{D_\theta}{G_\theta\sqrt K},
\quad
\eta_\mu=\frac{D_\mu}{G_\mu\sqrt K},
\]
where
\[
D_\theta:=\sup_{\theta,\theta'\in\Theta_\epsilon}\|\theta-\theta'\|_2,
\quad
D_\mu:=\sup_{\mu,\mu'\in\mathcal M}|\mu-\mu'|,
\]
we have
\[
\max_{\theta\in\Theta_\epsilon}\sum_{k=1}^K\psi_k(\theta)-\sum_{k=1}^K\psi_k(\theta^k)\le D_\theta G_\theta\sqrt K,
\]
and
\[
\max_{\mu\in\mathcal M}\sum_{k=1}^K q_k(\mu)-\sum_{k=1}^K q_k(\mu^k)\le D_\mu G_\mu\sqrt K.
\]
Consequently,
\[
\mathrm{Reg}_\theta(K)=O(\sqrt K),
\quad
\mathrm{Reg}_\mu(K)=O(\sqrt K).
\]
\end{lemma}

According to \Cref{lem:oco-dual-regret}, we can immediately obtain:
\[
\Reg = \OPT-\ALG=O(\sqrt K).
\]
This completes the proof.

\section{Proof of \Cref{lem:oco-dual-regret}}
\begin{proof}
We prove the two regret bounds using the standard analysis of projected online gradient ascent. The proof is included for completeness.

First consider the prompt-level dual variables. Let
\[
\Theta_\epsilon=\prod_{i=1}^M[\epsilon,c_i],
\]
and define its diameter
\[
D_\theta :=\sup_{\theta,\theta'\in\Theta_\epsilon}\|\theta-\theta'\|_2.
\]
The projected gradient ascent update is
\[
\theta^{k+1}=\Pi_{\Theta_\epsilon}\left(\theta^k+\eta_\theta\nabla\psi_k(\theta^k)\right).
\]
Fix any comparator \(\theta\in\Theta_\epsilon\). By non-expansiveness of Euclidean projection,
\[
\begin{aligned}
\|\theta^{k+1}-\theta\|_2^2
&\le\left\|\theta^k+\eta_\theta\nabla\psi_k(\theta^k)-\theta\right\|_2^2 \\
&=\|\theta^k-\theta\|_2^2+2\eta_\theta\left\langle\nabla\psi_k(\theta^k),\theta^k-\theta\right\rangle+\eta_\theta^2\|\nabla\psi_k(\theta^k)\|_2^2 .
\end{aligned}
\]
Rearranging gives
\[
\left\langle\nabla\psi_k(\theta^k),\theta-\theta^k\right\rangle\le\frac{\|\theta^k-\theta\|_2^2-\|\theta^{k+1}-\theta\|_2^2}{2\eta_\theta}+\frac{\eta_\theta}{2}\|\nabla\psi_k(\theta^k)\|_2^2 .
\]
Since \(\psi_k\) is concave, we have
\[
\psi_k(\theta)-\psi_k(\theta^k)\le\left\langle\nabla\psi_k(\theta^k),\theta-\theta^k\right\rangle .
\]
Combining the preceding two inequalities and using \(\|\nabla\psi_k(\theta^k)\|_2\le G_\theta\), we obtain
\[
\psi_k(\theta)-\psi_k(\theta^k)\le\frac{\|\theta^k-\theta\|_2^2-\|\theta^{k+1}-\theta\|_2^2}{2\eta_\theta}+\frac{\eta_\theta G_\theta^2}{2}.
\]
Summing over \(k=1,\ldots,K\), the squared-distance terms telescope:
\[
\begin{aligned}
\sum_{k=1}^K\left[\psi_k(\theta)-\psi_k(\theta^k)\right]
&\le\frac{\|\theta^1-\theta\|_2^2}{2\eta_\theta}+\frac{\eta_\theta G_\theta^2K}{2} \\
&\le\frac{D_\theta^2}{2\eta_\theta}+\frac{\eta_\theta G_\theta^2K}{2}.
\end{aligned}
\]
Choosing
\[
\eta_\theta=\frac{D_\theta}{G_\theta\sqrt K}
\]
yields
\[
\sum_{k=1}^K\left[\psi_k(\theta)-\psi_k(\theta^k)\right]\le D_\theta G_\theta\sqrt K.
\]
Since this holds for every \(\theta\in\Theta_\epsilon\), we have
\[
\max_{\theta\in\Theta_\epsilon}\sum_{k=1}^K\psi_k(\theta)-\sum_{k=1}^K\psi_k(\theta^k)\le D_\theta G_\theta\sqrt K.
\]

We next prove the corresponding bound for the global budget multiplier. Let
\[
\mathcal M=[0,\bar\mu],
\quad
D_\mu:=\sup_{\mu,\mu'\in\mathcal M}|\mu-\mu'|=\bar\mu .
\]
The projected gradient ascent update is
\[
\mu^{k+1}=\Pi_{\mathcal M}\left(\mu^k+\eta_\mu\nabla q_k(\mu^k)\right).
\]
Fix any comparator \(\mu\in\mathcal M\). By non-expansiveness of projection,
\[
\begin{aligned}
|\mu^{k+1}-\mu|^2
&\le\left|\mu^k+\eta_\mu\nabla q_k(\mu^k)-\mu\right|^2 \\
&=|\mu^k-\mu|^2+2\eta_\mu\nabla q_k(\mu^k)(\mu^k-\mu)+\eta_\mu^2|\nabla q_k(\mu^k)|^2 .
\end{aligned}
\]
Rearranging gives
\[
\nabla q_k(\mu^k)(\mu-\mu^k)\le\frac{|\mu^k-\mu|^2-|\mu^{k+1}-\mu|^2}{2\eta_\mu}+\frac{\eta_\mu}{2}|\nabla q_k(\mu^k)|^2 .
\]
Since \(q_k\) is concave,
\[
q_k(\mu)-q_k(\mu^k)\le\nabla q_k(\mu^k)(\mu-\mu^k).
\]
Using the uniform gradient bound
\[
|\nabla q_k(\mu^k)|\le G_\mu,
\]
we get
\[
q_k(\mu)-q_k(\mu^k)\le\frac{|\mu^k-\mu|^2-|\mu^{k+1}-\mu|^2}{2\eta_\mu}+\frac{\eta_\mu G_\mu^2}{2}.
\]
Summing over \(k=1,\ldots,K\) gives
\[
\begin{aligned}
\sum_{k=1}^K\left[q_k(\mu)-q_k(\mu^k)\right]
&\le\frac{|\mu^1-\mu|^2}{2\eta_\mu}+\frac{\eta_\mu G_\mu^2K}{2} \\
&\le\frac{D_\mu^2}{2\eta_\mu}+\frac{\eta_\mu G_\mu^2K}{2}.
\end{aligned}
\]
Choosing
\[
\eta_\mu=\frac{D_\mu}{G_\mu\sqrt K}
\]
yields
\[
\sum_{k=1}^K\left[q_k(\mu)-q_k(\mu^k)\right]\le D_\mu G_\mu\sqrt K.
\]
Since this holds for every \(\mu\in\mathcal M\), we obtain
\[
\max_{\mu\in\mathcal M}\sum_{k=1}^K q_k(\mu)-\sum_{k=1}^K q_k(\mu^k)\le D_\mu G_\mu\sqrt K.
\]
Therefore,
\[
\mathrm{Reg}_\theta(K)=O(\sqrt K),
\quad
\mathrm{Reg}_\mu(K)=O(\sqrt K).
\]
This completes the proof.
\end{proof}

\section{Experimental Details}
\label{app:exp-details}
In this section, we provide the training details for all four LLMs used in our experiments. All models are trained with \textsc{CERO} for $500$ optimization steps using the same OGD update rule. We use H800 GPUs for all runs: the 1.5B and 4B models are trained on $4$ GPUs, while the 7B model is trained on $8$ GPUs. 
R1-Distill-1.5B runs about $90$ GPU hours, Qwen3-4B runs about $170$ GPU hours and Qwen2.5-Math-7B runs about $200$ GPU hours.

We report OGD-related hyperparameters in \Cref{tab:experimental-details}. Across all models, we initialize $\mu$ and $\theta$ to $10^{-6}$ and $10^{-7}$, respectively. The OGD hyperparameter $\eta$ is set to $2 \times 10^{-3}$ for R1-Distill-1.5B and $1 \times 10^{-3}$ for the remaining models. We tune the learning rates for $\mu$ and $\theta$ separately for each model while keeping the remaining training protocol fixed.

\begin{table}[htbp]
\centering
\caption{
Training details for the four trained LLM.
}
\label{tab:experimental-details}
\resizebox{\linewidth}{!}{
\begin{tabular}{lcccccccc}
\toprule
Model
& average rollout $N$
& maximum rollout $\Nmax$
& $\mu$ learning rates
& $\theta$ learning rates
& Initial $\mu$
& Initial $\theta$
& $\eta$
\\
\midrule
R1-Distill-1.5B &  8 & 16 & 0.00003 & 0.00004 & 0.000001 & 0.0000001 & 0.002 \\
Qwen3-4B-base & 8 & 16 & 0.00005 & 0.00005 & 0.000001 & 0.0000001 & 0.001 \\
Qwen3-4B-Instruct& 8 & 16 & 0.00005 & 0.00005 & 0.000001 & 0.0000001 & 0.001 \\
Qwen2.5-Math-7B& 8 & 16 & 0.0001 & 0.00008 & 0.000001 & 0.0000001 & 0.001 \\
\bottomrule
\end{tabular}
}
\end{table}

\end{APPENDICES}

\end{document}